\documentclass[journal,letterpaper,twocolumn,twoside,nofonttune]{IEEEtran}

\usepackage[utf8]{inputenc} 
\usepackage[T1]{fontenc}
\usepackage{url}
\usepackage{ifthen}
\usepackage{cite}
\usepackage{amsmath}%
\usepackage{wasysym}%
\usepackage{amssymb}

\usepackage{mdwmath}
\usepackage{blindtext}
\usepackage{eqparbox}
\usepackage{fixltx2e}
\usepackage{stfloats}

\usepackage[english]{babel}
\usepackage{lipsum}
\usepackage{xcolor}
\usepackage{tikz}
\usepackage{graphicx}

\usepackage{algorithm, algorithmic}

\newtheorem{theorem}{{Theorem}}
\newtheorem{lemma}[theorem]{{Lemma}}
\newtheorem{proposition}[theorem]{{Proposition}}

\newtheorem{definition}{{Definition}}

\DeclareMathAlphabet{\mathbfsl}{OT1}{ppl}{b}{it} 

\newcommand{\mathset}[1]{\left\{#1\right\}}

\newcommand{\be}[1]{\begin{equation}\label{#1}}
\newcommand{\ee}{\end{equation}}

\renewcommand{\leq}{\leqslant}
 
\renewcommand{\geq}{\geqslant}

\renewcommand{\Bbb}{\mathbb}

\newcommand{\Cref}[1]{Co\-ro\-lla\-ry\,\ref{#1}}

\newcommand{\Ftwo}{{{\Bbb F}}_{\!2}}

\newcommand{\deff}{\mbox{$\stackrel{\rm def}{=}$}}

\newcommand{\bit}{\ensuremath{\mathset{0,1}}}

\usepackage[normalem]{ulem}   

\IEEEoverridecommandlockouts
\begin{document}

\title{Coding for Crowdsourced Classification\\ with XOR Queries}

 \author{%
   \IEEEauthorblockN{James (Chin-Jen) Pang, Hessam Mahdavifar, and S. Sandeep Pradhan \\ 
   }
   \IEEEauthorblockA{ Department of Electrical Engineering and Computer Science, University of Michgan, Ann Arbor, MI 48109, USA\\
                     Email: \{cjpang, hessam, pradhanv\}@umich.edu             
                 }
    \thanks{This work was supported by the National Science Foundation
under grants CCF--1717299 and CCF--1763348.}      

 }

\maketitle

\begin{abstract}
This paper models the crowdsourced labeling/ \hspace{-1mm}classification problem as a \textit{sparsely encoded} source coding problem, where each query answer, regarded as a code bit, is the XOR of a small number of labels, as source information bits. 
In this paper we leverage the connections between this problem and well-studied codes with sparse representations for the channel coding problem to provide querying schemes {with \textit{almost} optimal number of queries, each of which involving only a constant number of labels}. We also extend this scenario to the case where some workers can be unresponsive. For this case, we propose querying schemes where each query involves only $\log n$ items, where $n$ is the total number of items to be labeled. Furthermore, we consider classification of two correlated labeling systems and provide \textit{two-stage} querying schemes with \textit{almost} optimal number of queries each involving a constant number of labels. 
\end{abstract}

\section{Introduction}

\subsection{Crowdsourcing: Classification} \label{intro-clustering}
	Crowdsourcing is a  human-based problem-solving mechanism that allows a large crowd to distributively handle a massive number of queries. These problems, such as image classification, video 	annotation, form data entry, optical character recognition, translation, recommendation, and proofreading \cite{CrowdsourcingKargerNIPS2011, vesdapunt2014crowdsourcing}, typically require human involvement or suit human better than machines\cite{Varshney2012ParticipationIC}. In crowdsourcing systems, there are usually platforms, such as Amazon Mechanical Turk and Figure Eight, that match the taskmaster to a huge worker crowd. 
	However, the workers may be unreliable for several reasons: the reward for each task is usually as small as a few cents, the tasks are tedious, and one can still collect his rewards even if his answer is incorrect. 

    Many crowdsourcing-based real-life problems have one common goal deep down: classification/labeling of the items \cite{ipeirotis2010analyzing}.
    Formally, the label learning problem can be defined as follows: 
    suppose there are $n$ items, and the $i$-th item has a label $X_i \in \{ 0,1,...,L-1 \}$, for $i\in \{ 1,2,...,n \}$. The goal is to identify the labels of the items. This problem is equivalent to clustering $n$ items into $L$ clusters with ground truth. 
	Using crowdsourcing, a \textit{taskmaster} can send queries to workers (sometimes called \textit{oracles}, \textit{ human annotators }, or \textit{labelers}).  For instance, \textit{same cluster} queries are adopted in \cite{SameCluster_Hassan2016}, where in each query two items $u$ and $v$ are sent to a worker and the worker is asked "do $u$ and $v$ belong to the same label cluster?" 
	In \cite{BudgetClusteringwithSame-ClusterQueriesRowd2011} and \cite{CrowdsourcingKargerNIPS2011}, single-item queries are considered, where the worker receives an item for each query and answers a question, such as "This is a picture of a dog, true or false?". In another related work, each item is associated with a certain number of properties and workers are imperfect, and the goal is to identify an item by querying the workers the properties \cite{Reliable2014}. In general, in all these scenarios, the objective is to minimize the number of queries, for a given type of query,  sent to the workers, while being able to recover the labels with certain reliability/fidelity constraints. Note that, in practice, a certain number of queries can be assigned to one worker as a task, however, in the context of our paper, the goal is to minimize the number of queries regardless of how many workers are involved.

\subsection{Our Contributions}
	In this paper we consider the 2-cluster case, i.e., $L=2$, together with XOR queries, similar to the model adopted in \cite{semi-super}. 	 
	Unlike several prior works which considered queries involving only $1$ or $2$ items 
	\cite{CrowdsourcingKargerNIPS2011, SameCluster_Hassan2016}, we consider a generalized scenario, as considered in \cite{semi-super}, where the number of items involved in a query can be more than $2$ and up to a certain threshold. 
	Furthermore, besides the scenario where answers to queries are perfect, we consider an extension where some of the workers may not answer the queries assigned to them. 
	We first show that, in the scenario with perfect answers, the proof of \cite[Theorem 1]{semi-super} is incorrect, with the random ensemble adopted therein, and hence fails to support the theorem. 
	Note that the existential claim of the theorem may still be true if an alternative proof is given.
	We show that, with perfect workers, there exists a querying scheme within distance $\epsilon$ from the theoretic lower bound, in terms of the number of queries normalized by $n$, 
	with up to $\log \frac{1}{\epsilon}$ items in each query. A similar result is shown for the scenario with unanswered queries and with up to $\log n  \log \frac{1}{\epsilon}$ items in each query.
	
	We further extend the problem to the scenario when items need to be clustered according to two different clustering criteria. Ideas from correlated-source and channel coding allow us to recover two types of clustering with less queries than when the two types of labeling are recovered separately. We show that  a querying scheme within distance $2\epsilon$ from the theoretic lower bound, in terms of the number of queries normalized by $n$, is achievable with $\log \frac{1}{\epsilon}$ items per query. 

\section{Preliminaries}

\subsection{Crowdsourcing and Linear Channel Codes}   \label{LSC-intro}  

	  From crowdsourcing's perspective, when XOR query scheme is adopted, the process is similar to a linear source coding problem. In other words, the output of a query is the XOR of binary labels, i.e., addition over the binary field $\Ftwo$, of all items involved in that query. In particular, let $n$ be the total number of items to be labeled and $\textbf{X}=(X_1,X_2,\dots,X_n)^t \in \mathcal{X}^n$, where $\mathcal{X} =\{0,1\}$, considered as a column vector, denote the true labels, unknown to the taskmaster. Let also $A \in \Ftwo^{m \times n} $ denote the \textit{query matrix}, where ones in the $i$-th row of $A$ correspond to the indices of items in the $i$-th query. Under the XOR model for the query answers, the correct answer to the $i$-th query is equal to the $i$-th element of $A\textbf{X}$, where all operations are over $\Ftwo$, i.e., $X_i$'s are regarded as elements of $\Ftwo$. 
	  Given the limited capabilities of human workers, the queries need to be designed to be sparse. In other words, the number of items per query, which is equal to the number of ones in each row of $A$, must be small, e.g., bounded by a constant value. 

      It is common to assume the apriori distribution of the labels are known to the taskmaster \cite{CrowdsourcingKargerNIPS2011, mazumdar2016clustering}.
      In particular, an i.i.d. Ber$(p)$ distribution is assumed for the binary labels, where $p$ is known to the taskmaster. Note that due to the duality between source coding and channel coding problems, the parity-check matrix of a linear code $\mathcal{C}$ designed for transmission over a memoryless binary symmetric channel (BSC) with transition probability $p$, BSC$(p)$, can be used to compress an i.i.d. Ber$(p)$ source. In fact, the probability of error of the maximum-likelihood (ML) decoder would be the same when the code is used either for channel coding across BSC$(p)$ or source coding of an i.i.d. Ber$(p)$ source. Hence, low-density parity-check (LDPC) codes become relevant for the considered crowdsourced clustering problem with XOR query scheme and bounded number of items in each query.

\subsection{LDPC and LDGM Codes}  \label{intro_LDPC}
	LDPC codes were originally introduced by Gallager \cite{GallagerLDPC} in 60's and were later rediscovered in 90's \cite{MacKay1999} and were shown to offer near-capacity performance under practical belief-propagation (BP) decoding algorithms. 

	In \cite{GallagerLDPC}, Gallager proved that \textit{right-regular} LDPC codes, where each row of the parity-check matrix $H$ has a constant weight $\Delta$,
	 can not achieve the channel capacity on BSC. He also showed that the gap to capacity diminishes exponentially fast with $\Delta$. 
	In \cite{RichardsonEnsemble2001} and \cite{densityVSperfomance2003}, similar results are shown when the row weights and average row weights of $H$ are upper bounded by $\Delta$, respectively. 
	In terms of code constructions, several works on regular LDPC codes  have shown that rates close to the provided upper bounds are attainable \cite{densityVSperfomance2003,Burshtein2002regularLDPC,RateLDPC2002}. In section \ref{NlessQPR}, we leverage results from \cite{densityVSperfomance2003} and \cite{SCldpc2013} to obtain scheme with sparse queries in the context of crowdsourced classification. 
	
    When considering unreliable workers in the crowdsourcing setup another class of codes with sparse representations,	namely low-density generator-matrix (LDGM) codes, become relevant. In LDGM code, the generator matrix is assumed to be sparse. In general, as opposed to LDPC codes, the performance of LDGM codes has not been very well-studied in the literature. Several works have empirically shown the existence of LDGM codes with close-to-capacity performance \cite{ldgmno2,ldgmno4,ldgmno5}. 
    It is shown in \cite{LDGM_capAchieving2011} that ensembles of LDGM codes are capacity achieving over BSCs when the row weights scale linearly with $n$. Furthermore, row weights of $O(\log n)$ suffice to achieve the capacity of binary erasure channels (BECs) \cite{LDGM_capAchieving2011}. Scaling exponent of such codes were studied in \cite{mahdavifar2017scaling}.

\subsection{Prior Work}		\label{sec:prior}
	In \cite{semi-super}, three scenarios, namely noiseless queries and exact recovery, noiseless queries and approximate recovery, and noisy queries and approximate recovery, are discussed. Here, a query is noisy if the workers are unreliable, i.e., the query answers may be inaccurate, and is noiseless if it is always correct. A recovery is assumed to be \textit{perfect} if the (block) error probability vanishes as the number of items grows large, and is considered to be approximate if up to a constant probability of error in recovering labels/source bits is allowed.

	In \cite{semi-super}, Mazumdar and Pal attempt to show that, under the XOR query model with noiseless answers, perfect recovery of all labels (i.e. source compression) is achievable with sparsely encoded source coding, when the number of items in each query is bounded by a constant value $\Delta$. Note that in several other prior works, queries involving $\Delta = 1,2,3$ items were considered \cite{CrowdsourcingKargerNIPS2011,zhou2012learning, BudgetClusteringwithSame-ClusterQueriesRowd2011,SameCluster_Hassan2016,LimitBudgetFidelity2016,TriangularQueries}. 

	LDPC codes have been considered for source compression in \cite{CompressionLDPC2004} and \cite{LossyCompLDPC2003}. In particular, it is pointed out in \cite{CompressionLDPC2004} that the analogy between linear source codes and LDPC channel codes was largely neglected in the literature and it is shown that LDPC-based data compression for either memoryless sources or sources with memory are practical.

\section{Main Results}
    The main results of this paper are stated in this section. The proofs of Propositions 1 to 5 can be found in Section\,\ref{sec:appendix}.
\subsection{Noiseless Queries and Perfect Recovery}  \label{NlessQPR}
    \noindent\textbf{Problem Formulation:}
    We adopt the same model as in \cite[Section 3.1]{semi-super}. 
    Consider the crowdsourced classification problem where there are only two label types for each item and XOR queries are adopted, as discussed in Section\,\ref{LSC-intro}. Suppose that the labels $X_i \in \mathcal{X} = \{0, 1\}$, for $1 \leq i \leq n$, are i.i.d. Ber($p$) random variables. Without loss of generality, we may assume $p \in (0, 0.5)$. 
    In this subsection, we assume a nonadaptive/one-shot scenario, in which all queries are generated and sent to the workers at the same time. This is in accordance with several prior works, e.g. \cite{CrowdsourcingKargerNIPS2011, LimitBudgetFidelity2016}.

    In \cite[subsection 3.1]{semi-super}, the number of items in each query given to workers is fixed to $\Delta$, and the worker returns an error-free XOR of the labels in the given query. Let $H_b(p)$ denote the binary entropy function. The following result is claimed in \cite{semi-super}.

    \cite[Theorem 1]{semi-super}:
    \textit{
	    There exists a querying scheme with $$m= \frac{n(H_b(p)+o(1))}{\log_2 \frac{1}{\alpha}}$$ queries, of above type, where $\alpha \,\deff\,  \frac{1}{2} [ 1+ (1-4p(1-p))^\Delta ] $, that achieves perfect recovery.
    } 

    The proof provided in \cite{semi-super} is based on the following. The average probability of error of a randomly chosen query scheme, consisting of $m$ independently and uniformly selected random queries involving exactly $\Delta$ items, is analyzed and is claimed that it approaches $0$ as $n$ grows large. However, we show here that the provided proof does not hold. In particular, we show in Proposition 1 that the average probability of error over the considered ensemble is bounded away from $0$. 

    \begin{proposition} \label{Prop:XorInsufficient}
	    The average probability of error, denoted by $P_e$, in a scheme with $$m= \frac{n(H_b(p)+o(1))}{log \frac{1}{\alpha}} $$ independent and uniformly distributed random queries involving $\Delta$ items does not vanish as $n$, the number of items, grows . 
	    More precisely, 
	    \begin{equation} \label{eq:prop1}
	        P_e  \geq (1-\epsilon)(\exp(- \frac{\Delta \cdot  H_b(p)   }{log \frac{1}{\alpha}} ) -\epsilon ') > 0,
	    \end{equation}
	    where $\epsilon , \epsilon ' >0$ can be chosen arbitrarily small as  $ n \rightarrow \infty$
    \end{proposition} 	

    \noindent \textbf{Remark 1.} In \cite[Theorem 2]{CommuRecInHyper}, the authors consider uniformly random queries, essentially a random ensemble same as in Proposition 1, and show that $m= \Theta(n \log n)$ queries is necessary and sufficient for perfect recovery regardless of whether the workers are perfect or not, when no apriori distribution is assumed.

    This does not imply that the theorem itself, \cite[Theorem 1]{semi-super}, does not hold. Note also that such results are not about specific constructions and state that the average probability of error of certain random ensemble is bounded away from $0$. In fact, there are trivial cases of a query matrix providing perfect recovery with $m = n$ queries, e.g., the identity matrix. However, the question of whether less than $n$ queries, and more specifically close to $n H_b(p)$ queries, is sufficient for perfect recovery or not is not properly answered by \cite{semi-super} and has not been considered in \cite{CommuRecInHyper}. We answer this question in Proposition \ref{Prop:NoiselessRecovery}. 

    \noindent \textbf{Remark 2.} By Shannon's source coding theorem we need at least  $nH_b(p)$ queries to achieve perfect recovery, when the query answers are binary. 
    Hence, the compression rate, i.e., the number of queries normalized by the number of items, must be at least $H_b(p)$. In Proposition 2, it is shown that for any chosen $\epsilon \in (0,1)$, compression rate as small as $H_b(p)+\epsilon(1-H_b(p))$ can be achieved by a query scheme, with $O(\log \frac{1}{\epsilon})$ items per query, thereby providing a scheme with \textit{almost} optimal number of queries. 

    \begin{proposition} \label{Prop:NoiselessRecovery}
	    Suppose that workers are perfect and labels have prior distribution $Ber(p)$. Then,  for $\epsilon \in (0,1)$ and sufficiently large $n$, there exists a querying scheme using 
	    \begin{equation}\label{eq:prop2}
	        m= n[H_b(p)+\epsilon(1-H_b(p))]
	    \end{equation}
	    queries, each involving no more than $$({H_b(p)}^{-1} -1 )\frac{K_1     -K_2 ln{(\epsilon)}       }{1-\epsilon} $$ items, that achieves perfect recovery, where $ K_1 $ and $K_2$ depend only on $p$.  \label{NlessQnumberCor}
    \end{proposition}
    
\subsection{Two-label Perfect Recovery}\label{NlessTwoSource}
    In this subsection, we extend the problem discussed in Section\,\ref{NlessQPR} to the recovery of two different clustering based on two properties/labeling criteria associated with the same set of objects. In particular, suppose that the $i$-{th} item can be classified according to two (binary) labeling systems and is labeled $X_i, Y_i$, respectively. For instance, one labeling system may involve identifying the objects in pictures, e.g., whether there is a cat or dog in the picture, and the other may involve identifying the location, e.g., whether this picture is taken indoor or outdoor. Furthermore, we assume that a two-stage query scheme is adopted, where queries involving $X_i$'s are sent in the first stage and queries involving $Y_i$'s are sent in the second stage.  

    By leveraging the correlation between labels $X$ and $Y$, we can recover the clustering by sending an almost optimal number of queries, under the constraint of finite items per query. 
 
    \begin{proposition}\label{Prop:correlatedNoiseless}
        Let  $(X_i,Y_i) \stackrel{i.i.d.}{\sim}P_{X,Y}(x,y)$, for $1 \leq i \leq n$. Then there exists a two-stage querying scheme using 
        \begin{equation}\label{eq:prop3}
            m= n(H(X,Y)+2\epsilon)
        \end{equation}
        queries, where $H(X,Y)$ is the joint entropy function, each involving no more than   $ O{( \log \frac{1}{\epsilon})} $  items, that achieves perfect recovery for sufficiently large $n$.   
    \end{proposition}

\subsection{Noisy Queries and Perfect Recovery}  \label{NsyQPR}
    In this subsection, we consider imperfection in workers' replies.
    Due to the monotonicity of queries, low payment for completion of queries, or the fact that workers may not be {experts}, two noisy scenarios for the answers to queries can emerge. 
    In the first case, it is assumed that some queries are not replied within a certain specified response time, or not replied at all. 
    This scenario can be modeled as follows: each query is replied (correctly) with probability $1-r$ and is not replied with probability $r$, for a certain parameter $r$, independent from other queries. Consequently, this scenario becomes related to the channel coding problem over the binary erasure channel (BEC) with erasure probability $r$. In the second case, it is assumed that the workers have accuracy $1-q$, that is, the answer from a worker is correct with probability $1-q$. This is related to the channel coding problem over BSC$(q)$. 

    In this paper, we focus on the first case and leave the second case for future work. In particular, we show that concatenation of LDGM and LDPC codes achieve compression rate 
    $$ R = [H_b(p)+\epsilon(1-H_b(p))]/(1-r), $$
    with row weights upper bounded by $\Delta = O(\log \frac{1}{\epsilon}\log n)$. 

    Let $\mathcal{A}_{N\times K}$ denote the set of $N\times K$ binary matrices and let $B(\mathcal{A}_{N\times K},p)$ denote a distribution on $\mathcal{A}_{N\times K}$, where the entries of a random $A \sim B(\mathcal{A}_{N\times K},p)$ are distributed i.i.d. with Ber$(p)$. 
    We utilize the following result from \cite{LDGM_capAchieving2011} to derive the main result of this section.

    \cite[Theorem 5]{LDGM_capAchieving2011}: Consider BEC$(r)$ and let $K = NR$, where $R < 1-r$. Suppose $A \sim B(\mathcal{A}_{N\times K},\rho(N))$, where $\rho(N) = \Theta(\frac{\log N}{N})$. Let $p_c(A)$ denote the probability of correct decoding, under ML, using $A^t$ as the generator matrix and assuming transmissions over BEC$(r)$. Then
    \begin{equation}\label{eq:LGDMpaperThm}
        \lim_{n\rightarrow \infty} \mathbb E_{A}(p_c(A)) =1 ,
    \end{equation} 	 where the expected value is taken with respect to $A$. 

    Note that the generation of $A \sim B(\mathcal{A}_{N\times K},\rho(N))$ does not guarantee that all row weights are bounded by $\log N$. We extend the result of \cite[Theorem 5]{LDGM_capAchieving2011}, in order to ensure that all row weights are bounded, in Proposition \ref{Prop:LDGM_sparse_Ensemble}.

    \begin{proposition}\label{Prop:LDGM_sparse_Ensemble}
		 Let $A \in \mathcal{A}_{N\times K}$ and let $A^t$ be the generating matrix corresponding to a code of rate $R = \frac{K}{N}<1-r$ with transmissions over BEC($r$).
		 For any $\rho(N) = \Theta(\frac{\log N}{N})$, the expected value of $p_c(A)$ over all matrices with  $\tilde{B}(\mathcal{A}_{N\times K},\rho(N))$  distribution tends to $1$ as $N$ approaches infinity, i.e., 
        \begin{equation}\label{eq:prop4}    	
            \lim_{N\rightarrow \infty} \mathbb E_{A\sim \tilde{B}}(p_c(A)) =1 ,
        \end{equation}
	    where  $\tilde{B}(\mathcal{A}_{N\times K},\rho (N))$ is obtained from the distribution $B(\mathcal{A}_{N\times K},\rho (N)))$ by removing matrices that have at least one row with Hamming weight larger than or equal to $ \Theta(\log N) .$
    \end{proposition}

    \begin{theorem}\label{Thm:m_forBEC}
	    Suppose that queries are answered with probability $1-r$, independent to each other. Then there exists a query scheme with
        \begin{equation}\label{eq:thmBEC}
            m = n[H_b(p)+\epsilon(1-H_b(p))]/(1-r)
        \end{equation}	
		queries, each involving $O(  \log \frac{1}{\epsilon} \log n) $ items, that guarantees perfect recovery of the labels as $n$ grows large. 
    \end{theorem}
    
    \begin{proof}		
	    The following statements holds for sufficiently large $n$.
	    The parity-check matrix of LDPC codes is applied first for the compression. By Proposition \ref{Prop:NoiselessRecovery} there exists an $m_H(n) \times n$ binary matrix $H_n$, where $m_H(n)= n[H_b(p)+\epsilon(1-H_b(p))]$, with row weights bounded from above by {$ ({H_b(p)}^{-1} -1 ) $} $\frac{K_1 - K_2 \ln\frac{1}{\epsilon}}{1-\epsilon} $. Note that $H_n$ can be used to compress an i.i.d. $Ber(p)$ source sequence with perfect recovery. 
		 
		Then, the compressed bits are further encoded to recover from erasures caused by the BEC$(r)$. By Proposition 4, there exists a $m_G(n) \times m_H(n)$ matrix $G_n$, with all row weights being $O(\log m_H(n))$, for transmission of $m_H(n)$ bits over BEC($r$), while the expected probability of correct decoding approaches $1$ as $n\rightarrow \infty$ for any $R_c\,\deff\, m_H(n)/m_G(n) < 1-r $.
		 
		Next, the two linear operations, one for the compression and the other one for the erasure correction, are concatenated, leading to the overall query matrix. Let $G^{SC}(n) \,\deff\, G_n H_n$, with dimensions $m_G(n) \times n$, denote the corresponding overall encoding matrix. Note that all row weights of $G^{SC}(n)$ are bounded from above by 
		$$
		({H_b(p)}^{-1} -1 )  \frac{K_1 -K_2 \ln\frac{1}{\epsilon}}{1-\epsilon} O(\log m_H(n)) = O(\log \frac{1}{\epsilon} \log n), 
		$$ and also the number of queries is equal to $m_G(n)$, where
		\begin{align*}
		     m_G(n) &> m_H(n)/(1-r) \\
		     &= n[H_b(p)+\epsilon(1-H_b(p))]/(1-r),
		\end{align*}
		is sufficient to show that $G^{SC}(n)$ guarantees perfect recovery of the labels as discussed next. Let $\textbf{X}$ denote the vector of $n$ labels, $\textbf{Y} =H_n \textbf{X}$ denote the compressed labels, and $\textbf{Z}= G_n \textbf{Y}  =G^{SC}(n) \textbf{X}$ and the taskmaster collects $\textbf{Z}$ corrupted with erasures with probability $r$. The taskmaster can recover $\textbf{Y}$, with high probability, by the choice of $G_n$. Then, having recovered $\textbf{Y}$, perfect recovery of $\textbf{X}$ is possible by the choice of $H_n$. That completes the proof of theorem.
\end{proof}

\section{Proofs} \label{sec:appendix}
\subsection{Proof of Proposition \ref{Prop:XorInsufficient}}\label{pf:prop:xorInsufficient}
	We analyze the average probability of error, $P_e$, of the querying scheme given in \cite[Theorem 1 ]{semi-super}, and provide a lower bound for it. 
	Consider the following two typical sets in $\bit^n$, 
	\[ A_\epsilon ^n (X)\deff\{ x^n  : np(1-n^ {- \frac{1}{3} } ) \leq  w_H(x^n)  \leq np(1+n^ {- \frac{1}{3} } )  \}, 	\] 
	\[
	B_\epsilon ^n (X)\deff\{ x^n\hspace{-1mm}:\hspace{-0.5mm}np(1-n^ {- \frac{1}{3} } )+1 \hspace{-0.5mm}\leq\hspace{-0.5mm}  w_H(x^n)  \hspace{-0.5mm}\leq \hspace{-0.5mm} np(1+n^ {- \frac{1}{3} } )-1\}, \]
	where $w_H(x^n )$ denotes the hamming weight of vector $x^n$. Then, we have
	\[ \Pr(A_\epsilon ^n (X) ) \rightarrow 1 \text{ and } \Pr(B_\epsilon ^n (X) ) \rightarrow 1 \text{ as } n \rightarrow \infty.  \]
	Note that for $x^n \in B_\epsilon ^n (X)$, 
	$$  np(1-n^ {- \frac{1}{3} } ) \leq  w_H(x^n +e )  \leq np(1+n^ {- \frac{1}{3} } ),	$$ 
	for any $e \in  \{ 0,1 \} ^n \text{ with unit weight} $, i.e., $ w_H(e) = 1$. Hence, $x^n +e \in  A_\epsilon ^n (X),$ for any $x^n \in B_\epsilon ^n (X)$. Then $P_e$ can be expressed and lower bounded as follows:
	\begin{align*}
	    &P_e  = \sum_{x^n \in \mathcal{X}^n} P_X^n(x^n)\Pr(\hat{x}^n \neq x^n) \\
	    & =  \sum_{x^n \in \mathcal{X}^n} P_X^n(x^n) \Pr( \exists  \tilde{x}^n \neq x^n, \tilde{x}^n \in A_\epsilon ^n (X), Qx^n= Q\tilde{x}^n ) \\
	    & \geq \hspace{-4mm}\sum_{x^n \in  B_\epsilon ^n (X)} P_X^n(x^n) \Pr( \exists  \tilde{x}^n \neq x^n, \tilde{x}^n \in A_\epsilon ^n  (X), Qx^n= Q\tilde{x}^n ) \\
	    & \geq \hspace{-4mm}\sum_{x^n \in  B_\epsilon ^n (X)} \hspace{-2mm} P_X^n(x^n) \Pr( x^n+e_1 \in A_\epsilon ^n  (X), Qx^n= Q ( x^n+e_1 )  ), 
	\end{align*}
	where  $e_1 = (1, 0 ,0 ..., 0,  0) $ and $ Q $ is the $m \times n$ query matrix corresponding to the query scheme described in Proposition\,1.
	
	For $x^n \in  B_\epsilon ^n (X)$, we always have $x^n+e_1 \in A_\epsilon ^n  (X) $. Note that $ Qx^n= Q ( x^n+e_1 )$ if and only if  $ Q ( e_1 ) = 0  $, which is also equivalent to the first label not being queried by any of the $m$ queries. The probability that a query does not use the first label is $\frac { \binom{n-1}{\Delta} }{  \binom{n}{\Delta}  }  = 1-\frac{\Delta}{n}$. 
	Thus, 	$$	\Pr \big(  Q ( e_1 ) = 0  \big) =  (1-\frac{\Delta}{n})^m,	$$	where $m= \frac{n(H_b(p)+o(1))}{log \frac{1}{\alpha}}$ as in \cite[Theorem 1]{semi-super}. Then we have 
	\begin{align*}
	    \Pr \big(  Q ( e_1 ) = 0 \big) 
	    &=  [(1-\frac{\Delta}{n})^n]^\frac{H_b(p)+o(1)}{log \frac{1}{\alpha}}\\ 
	    &\xrightarrow[n \to \infty]{} \exp(- \frac{\Delta  \cdot  H_b(p)}{log \frac{1}{\alpha}} ) > 0 .
	\end{align*}
	Therefore, \begin{align*}
	    P_e & \geq  \sum\nolimits_{x^n \in  B_\epsilon ^n (X)} P_X^n(x^n) (1-\frac{\Delta}{n})^m 
	    \\  & \geq (1-\epsilon)(\exp(- \frac{\Delta  \cdot  H_b(p) }{log \frac{1}{\alpha}} ) -\epsilon ') \neq 0,
	\end{align*}
	for  $n$ sufficiently large, where $\epsilon , \epsilon'$ can be chosen arbitrarily small as $n \rightarrow \infty$.
	
\subsection{Proof of Proposition \ref{Prop:NoiselessRecovery} } \label{Pf:prop:noiselessRecovery}
    We leverage a result from \cite{densityVSperfomance2003}, which establishes a connection between sparsity of the parity-check matrix and reliability performance of LDPC codes. First, the \textit{density} of parity-check matrix of a linear code is defined as follows.
    \begin{definition}
	    Given an $m \times n$ parity-check matrix $H$, the \textit{density} of $H$, denoted by $\rho = \rho(H)$, is the number of ones in $H$ normalized by $n$, i.e., the total number of ones in $H$ is $(n-m)\rho$.
    \end{definition}
    \cite[Theorem 2.2]{densityVSperfomance2003}: 
    For any BSC or BEC, there exists a sequence of ensembles of regular LDPC codes which achieves, under ML decoding, a fraction $1-\epsilon$ of the channel capacity with vanishing block error probability, and the asymptotic density of their parity-check matrices satisfles

    \begin{equation} \label{eq:LDPCdensity}
        \lim_{n\rightarrow \infty} \rho_n \leq \frac{K_1 -K_2 ln{(\epsilon)} }{1-\epsilon}, 
    \end{equation}
    where $K_1$ and $K_2$  depend only on the channel. In particular, for $BSC(p)$ with capacity $1-H_b(p)$, there exists an ensemble of regular $(n, n-m)$ LDPC codes achieving channel rate $$ R= (n-m)/n = (1-\epsilon)C =  (1-\epsilon)(1-H_b(p) ) .$$ The parity check matrices from this ensemble have $ m =n[H_b(p) +\epsilon(1-H_b(p))] $ rows, each with weight 
    \begin{align}
        (n-m)\rho_n/m &=   ([{H_b(p)+\epsilon(1-H_b(p)) }]^{-1}  -1)\rho_n \nonumber \\
        &\leq ({H_b(p)}^{-1} -1 )\rho_n \nonumber \\
        &\leq  ({H_b(p)}^{-1} -1 )\frac{K_1 -K_2 ln{(\epsilon)} }{1-\epsilon},  
    \end{align} for sufficiently large $n$.

    As mentioned in section \ref{LSC-intro}, we note that the parity-check matrix of a linear code $\mathcal{C}$ designed for transmission over a BSC$(p)$, can be used to compress an i.i.d. Ber$(p)$ source with the same block error probability $P_e$ under maximum likelihood (ML) decoding. 

\subsection{Proof of Proposition \ref{Prop:correlatedNoiseless}} \label{Pf:prop:correlatedNoiseless}

 	We use a two-stage querying scheme. 
 	Let $X_i\stackrel{i.i.d.}{\sim}Ber(p)$,  $q= \Pr(Y_i= 1|X_i=1)$ and $r= \Pr(Y_i=1| X_i=0)$. 

 	First, we use 
 	$$  	m_1 = n[H_b(p)+\epsilon(1-H_b(p))]  	$$ queries to retrieve $X_i$'s. 
 	From Proposition \ref{Prop:NoiselessRecovery}, 
 	$({H_b(p)}^{-1} -1 )\frac{K_1     -K_2 ln{(\epsilon)}       }{1-\epsilon} $ items per query  	are sufficient for perfect recovery of the  $X_i$'s. 
 	
 	Second, we consider the conditional distribution of the $Y_i$ labels. Suppose that $X_i$'s are recovered correctly. 
 	Let $n_1= \sum_{i=1}^{n}{X_i}$ and $n_2 = n- n_1$ denote the number of items with labels $X_i =1$ and $X_i =0$, respectively.
 	By law of large numbers, for any given $\epsilon'$, we have $n_1 \leq np(1+\epsilon')$ and $n_2 \leq n(1-p)(1+\epsilon')$, with high probability, for sufficiently large $n$. 
 	By Proposition \ref{Prop:NoiselessRecovery}, there exists a querying scheme with 
 	$$  	m_2 = n_1[H_b(q)+\epsilon(1-H_b(q))]  	$$ queries, each involving no more than 
 	$({H_b(q)}^{-1} -1 )\frac{K_1'     -K_2' ln{(\epsilon)}       }{1-\epsilon} $ items that recovers the $Y_i$ labels for the items  labeled $X_i =1$. 
 	Also, there exists a querying scheme with 
 	$$ m_3 = n_2[H_b(r)+\epsilon(1-H_b(r))] $$ queries, each involving no more than	 $({H_b(r)}^{-1} -1 )\frac{K_1''     -K_2'' ln{(\epsilon)}       }{1-\epsilon} $ items that recovers the $Y_i$ labels for the items  labeled $X_i =0$. 
 	The total number of queries is then
 	\begin{align*}
 	    m_1 +m_2 +m_3  &\leq n[H_b(p)+pH_b(q)+(1-p) H_b(r)]+\\
 	    &n\epsilon(1+p+(1-p))- \\
 	    &n[\epsilon H_b(p)+p\epsilon H_b(q)+(1-p)\epsilon H_b(r))-\epsilon' ]  	\\
	    &\leq  	nH(X,Y)+2n\epsilon. 
 	\end{align*}
 
    Note that, given $p,q,r$, the number of items involved in a query is related to the gap to capacity as $O(\log(\frac{1}{\epsilon}) ) $.

\subsection{Proof of Proposition \ref{Prop:LDGM_sparse_Ensemble}} \label{Pf:prop:LDGM_sparseEnsemble}
    The following lemma is a direct consequence of Chernoff bound and the proof is omitted here:
 	\begin{lemma} \label{Lemma:chernoffSumRVs}
 		Let $X_1, X_2, ...,X_N $ be $N$ independent random variables with $ X_i \sim \text{Ber}(p_i)$. Let $\mu = \sum_{i=1}^{N}p_i$. 
 		Then 
 		\begin{equation}
 		    \Pr(\sum\nolimits_{i=1}^{N}X_i  \geq (1+\delta)\mu) \leq e^{-\frac{\delta^2}{2+\delta}\mu} , \forall \delta>0. 
 		\end{equation} 		
 	\end{lemma}
 
    For any $\rho (N)=\Theta(\frac{\log N}{N})$, there exist $M > 1, m>0$ such that  $m\frac{\log N}{N}\leq   \rho (N) \leq M\frac{\log N}{N}$ for $N$ sufficiently large. 
    If $p = p_1 = p_2= ...= p_N = \rho(N)$, we have 
    $$ m\log N \leq \mu = Np \leq M \log N . $$
    Then, 
    \begin{align}
 	    \Pr(\sum_{i=1}^{N}X_i  \geq \delta(m)M\log N)& \leq	\Pr(\sum_{i=1}^{N}X_i  \geq \delta(m)\mu ) \nonumber \\
 	    &\leq e^{-2\log N} = N^{-2}, 
 	\end{align}
    where $\delta(m) >1$ is chosen such that $ \frac{(\delta(m)-1)^2}{2+\delta(m)-1}\times m >2$.

    Next, we discuss the probability that a matrix $A \sim B(\mathcal{A}_{N\times K},\rho(N))$ has \textit{heavy} rows, where a heavy row is a row with weight larger or equal to $\delta(m)M\log N$. 
    \begin{align*}
	    &\Pr(\mbox{each row of }A \mbox{ has weight less than } \delta(m)M\log N) \\
	    &= 1- \Pr( \bigcup\nolimits_{i=1}^{N}\{i^{th} \mbox{ row has weight } \geq \delta(m)M\log N\}) \\
	    & \geq 1- \sum\nolimits_{i=1}^{N}\Pr(i^{th} \mbox{ row has weight  }\geq \delta(m)M\log N\})\\
	    & = 1- N \cdot \Pr(1^{st} \mbox{ row has weight }\geq \delta(m)M\log N\})\\
	    & \geq 1- N\cdot N^{-2} \rightarrow 1 \mbox{ as } N \rightarrow \infty.
    \end{align*}
    Let $p_h(N)$ denote the probability that $A $ has at least one heavy row. Then we have $p_h(N) \rightarrow 0$. Note that
    \[ 
	 \mathbb E_{A}(p_c(A)) = p_h(N) \mathbb E_{A}(p_c(A)) +(1-p_h(N)) \mathbb E_{A\sim \tilde{B}}(p_c(A)),
    \] where the expectations are taken with $B(\mathcal{A}_{N\times K},\rho(N))$ distribution, $B(\mathcal{A}_{N\times K},\rho(N))$ distribution and heavy-row condition, and $\tilde{B}(\mathcal{A}_{N\times K},\rho(N))$ distribution. 

    From  \cite[theorem 5]{LDGM_capAchieving2011}, 
    \begin{align*}
        1= & \lim_{N\rightarrow \infty} \mathbb E_{A\in \mathcal{A}_{N\times K}}(p_c(A))   \\
        =&\lim_{N\rightarrow \infty} [ p_h(N) \mathbb E_{A}(p_c(A)) +(1-p_h(N)) \mathbb E_{A\sim \tilde{B}}(p_c(A))  ]\\
        =&\lim_{N\rightarrow \infty} \mathbb E_{A \sim \tilde{B}}(p_c(A))
    \end{align*}

\section{Conclusion}
    In this paper, crowdsourced classification problems with XOR querying schemes involving a limited number of items in each query sent to the crowd workers are considered. The goal is to perform the classification efficiently, that is, to minimize the number of queries sent to workers. We discuss the scenario where all workers are perfect, and then extend to the case with possibly having unresponsive workers. We further consider clustering the items based on two correlated classification criteria. 
    In all of the above cases, we provide querying schemes with almost optimal number of queries each with limited number of items. 

    There are several directions for future work. In this paper, we focus on binary labels and generalizing the results to classification problems with more than two possible labels is an interesting direction. To this end, one needs to consider non-binary codes with sparse representations. When considering noisy queries, only the scenario with unresponsive workers are studied. As stated in Section\,\ref{NsyQPR}, it is possible that workers do not provide the correct answer. In the binary labeling case, such scenario corresponds to coding over BSC. Hence, designs of \textit{good} LDGM codes for transmission over BSCs are needed for such classification problems with unreliable workers, which is another direction for future work. Furthermore, alternative methods, such as a joint source-channel coding design, can be utilized and may lead to more efficient querying schemes involving unreliable/unresponsive workers. 

\bibliographystyle{IEEEtran}
\bibliography{IEEEabrv}

\end{document}